\documentclass[journal]{IEEEtran}

\usepackage[utf8]{inputenc} % allow utf-8 input
\usepackage{hyperref}       % hyperlinks
\usepackage{url}            % simple URL typesetting
\usepackage{booktabs}       % professional-quality tables
\usepackage{amsfonts}       % blackboard math symbols
\usepackage{nicefrac}       % compact symbols for 1/2, etc.
\usepackage{microtype}      % microtypography
\usepackage{xcolor}         % colors
\usepackage{algorithm}      % algorithms
\usepackage{algpseudocode}  % algorithms

\usepackage{tikz}

\usepackage{amsmath}
\usepackage{mathtools}

\renewcommand{\[}{\left[}
\renewcommand{\]}{\right]}

\newcommand{\ceil}[1]{\left\lceil#1\right\rceil}

\newcommand{\doubleR}{\mathbb{R}}

\newcommand{\scriptE}{\mathcal{E}}

\newcommand{\scriptN}{\mathcal{N}}
\newcommand{\scriptO}{\mathcal{O}}

\newcommand{\scriptX}{\mathcal{X}}
\newcommand{\scriptY}{\mathcal{Y}}

\DeclareMathOperator*{\E}{\mathbb{E}}

\renewcommand{\d}[1]{\;{\rm d}{#1}}

\renewcommand{\=}{\coloneqq}

\usepackage{amsthm}
\newtheorem{theorem}{Theorem}

\newtheorem{lemma}{Lemma}

\theoremstyle{definition}

\usepackage{pdflscape}

\begin{document}

\title{Learning minimal volume uncertainty ellipsoids}

\author{Itai Alon, David Arnon and Ami Wiesel
\thanks{Manuscript received October, 2023.}
\thanks{The authors are with the School of Computer Science and Engineering,
The Hebrew University of Jerusalem, Jerusalem 9190401, Israel (e-mail: itai.alon1@mail.huji.ac.il; amiw@cs.huji.ac.il).}
\thanks{The second author is with Rafael Advanced Defence Systems - Research, Development and Engineering, Israel (e-mail: david.arnon@gmail.com).}
\thanks{This research was partially supported by ISF grant number 2672/21.}}

% \markboth{Journal of \LaTeX\ Class Files, Vol. 14, No. 8, August 2015}
% {Shell \MakeLowercase{\textit{et al.}}: Bare Demo of IEEEtran.cls for IEEE Journals}
\maketitle

\begin{abstract}
    We consider the problem of learning uncertainty regions for parameter estimation problems. The regions are ellipsoids that minimize the average volumes subject to a prescribed coverage probability. As expected, under the assumption of jointly Gaussian data, we prove that the optimal ellipsoid is centered around the conditional mean and shaped as the conditional covariance matrix. In more practical cases, we propose a differentiable optimization approach for approximately computing the optimal ellipsoids using a neural network with proper calibration. Compared to existing methods, our network requires less storage and less computations in inference time, leading to accurate yet smaller ellipsoids. We demonstrate these advantages on four real-world localization datasets.
\end{abstract}


\IEEEpeerreviewmaketitle

\section{Introduction}

We consider uncertainty quantification for parameter estimation. Estimating unknown parameters given noisy observations is a fundamental problem in statistical signal processing, including inverse problems and localization \cite{kay1993fundamentals}. 
In critical decision-making tasks, quantifying the uncertainty associated with these estimates is crucial. Accurate uncertainty regions are also necessary for developing more robust signal processing algorithms \cite{elbir2023twenty,luo2022efficient,zoubir2018robust}.
Traditionally, the estimation algorithms and their uncertainty regions have been designed and analyzed based on statistical models. Recently, there has been a switch to data-driven methods based on machine learning \cite{shlezinger2023model,khobahi2021lord,diskin2023learning,dahan2023uncertainty}. This letter focuses on deep learning solutions to uncertainty quantification as the natural next step.

The standard approach to uncertainty quantification is to associate an uncertainty region, typically an ellipsoid, around each estimate. Ideally, the ellipsoids should have small volumes and cover the true unknown parameters. The goal is to trade off these two competing properties optimally. Classically, in model-based estimation, the shape of the ellipsoid is characterized by the likelihood, asymptotic covariance matrix, and Fisher Information matrix. The volume is traditionally computed assuming a multivariate Gaussian distribution \cite{censi2007accurate,pascal2008performance}. Linearization methods can address more complex uncertainty sources \cite{brossard2020new}. Other approaches rely on perturbations of the observations and an analysis of the resulting deviations \cite{buch2017prediction,zambianchi2017distributed}.

Recently, there has been a switch to data-driven estimation based on fitting an algorithm concerning a training set. In uncertainty quantification, the main idea is to compare the current measurement to similar examples in this set and rely on their uncertainties. Following this line of thought, many works are variants of nearest neighbors algorithms \cite{wang2002nearest, johnstone2021conformal}. The successful CELLO approaches rely on fast estimating the covariance matrices of the errors \cite{vega2013cello, landry2019cello}. In particular, uncertainty methods typically allow imperfect coverage and a small percentage of examples outside the uncertainty region. In this sense, the problem is also related to a large body of literature on robust covariance estimation  \cite{ollila2020shrinking,ollila2014regularized,sun2016robust,pascal2008performance,fishbone2023highly}.

Most of the previous data-driven methods lack formal guarantees and do not always satisfy the required coverage in practice. To close this gap, conformal prediction is a distribution-free approach that uses an independent calibration phase that provides validity for any desired coverage level \cite{vovk1999machine,saunders1999transduction, vovk2005algorithmic,angelopoulos2021gentle}. 
Cross-conformal prediction \cite{vovk2015cross} and cross-validation \cite{barber2020distribution} generalize these ideas and utilize all the available data for fitting and calibration.
Initially, these ideas were focused on classification problems and scalar regressions. Recent advances also address the multivariate case by marginal intervals that result in rectangular uncertainty regions \cite{messoudi2021copula}. Closest to our work is the extension of conformal prediction to calibration of ellipsoid uncertainty regions based on nearest neighbors \cite{feldman2021calibrated,messoudi2022ellipsoidal}. 

We present a unified Learning Minimum Volume Ellipsoids (LMVE) framework that considers the volume vs. coverage tradeoff. We demonstrate that the optimal shapes of the ellipsoids are conditional covariances in the Gaussian case. We propose a deep learning solution for more realistic settings that combines nearest-neighbor approaches, covariance estimation, and conformal prediction. The proposed LMVE neural network generates ellipsoids with minimal average volume and prescribed coverage probability. It significantly reduces memory and computation resources, improves accuracy, and is demonstrated on four real-world localization datasets.

The fully implemented code can be found in \href{https://github.com/ItaiAlon/LMVE/}{https://github.com/ItaiAlon/LMVE/}.

\section{Problem Formulation}

Let $x \in \doubleR^d$ and $y \in \doubleR^n$ be a features vector and a labels vector drawn from a joint distribution $\scriptX \times \scriptY$. An uncertainty ellipsoid on $y$ given $x$ is defined as:
\begin{equation}
    \scriptE(\mu(x),C(x)) = \{ y : (y - \mu(x))^T{C}^{-1}(x)(y - \mu(x))) \leq 1 \}
\end{equation}
where $\mu(x)$ is the center of the ellipsoid and $ C(x)\succ 0$ characterizes its shape. Generally, ellipsoids are defined by both the center and the shape. Still, in many applications, assuming that the center is fixed and given from a previous prediction stage (either model-based or data-driven) is more reasonable. Therefore, most of this letter will focus on the shape $C(x)$.

Ideally, the ellipsoids should satisfy two competing properties. First, to ensure validity, each ellipsoid should satisfy the coverage constraint defined as 
\begin{align}
    \Pr\left[y\in {\scriptE}(\mu(x),C(x)) \right]\geq \eta
\end{align}
% \begin{align}
%     {\E}[{\rm{Err}}_{\scriptE}(x,y)] \le \varepsilon
% \end{align}
% with the indicator function
% \begin{align}\label{indic}
%     {\rm{Err}}_{\scriptE}(x,y) = \begin{cases}
%         0 & y \in {\scriptE}(\mu(x),C(x)) \\
%         1 & {\rm{else}}
%     \end{cases}
% \end{align}
and $\eta \in [0,1)$ a prescribed coverage probability. Second, it should minimize the uncertainty defined by the volume
\begin{equation}
        {\rm{Vol}}_{\scriptE}(x) = {\rm const} \cdot \det({C}(x))^\frac12
\end{equation}
where the constant is the volume of the unit ball in that dimension. 
To formulate this tradeoff, we use a constrained optimization and minimize the expected volume subject to a coverage constraint
\begin{align}
    \label{eq: MVE}
     (MVE): \;\min_{{\scriptE}(\mu(\cdot),C(\cdot))} &\quad \E[{\rm{Vol}}_{\scriptE}(x)] \nonumber\\
    \rm{s.t.}\;\;\;\; & \Pr\left[y\in {\scriptE} \right]\geq \eta.
\end{align}

For completeness, our uncertainty ellipsoids assume a Bayesian setting in which both $x$ and $y$ are random, and the expectations are computed. This differs from the non-Bayesian definition of confidence regions, which assumes that $y$ are deterministic unknown and will be addressed in future work.

The main goal of this paper is to develop a method for approximately solving \hyperref[eq: MVE]{(MVE)} using a neural network. At inference time, it will take $x$ as an input and output a center $\mu(x)$ and a shape matrix $C(x)$ that defines the uncertainty ellipsoid associated with the corresponding unknown $y$. We assume access to independent and identically distributed (i.i.d.) pairs of examples $(x_i, y_i)$ for $i=1\cdots,m$ to train this network. 

\section{Theory}

In this section, we theoretically analyze the optimal solution to \hyperref[eq: MVE]{(MVE)} in the simplistic Gaussian setting. 

We begin by considering a single ellipsoid (without any features $x$). This result embodies a widely recognized conventional wisdom, yet we are unaware of a previous formal proof.

\begin{lemma}
    \label{lemma:positive_proportion}
    Let $\eta \in [0,1]$ and $y \sim \scriptN(\mu, \Sigma)$ with dimension $n$.  Consider the optimization
    \begin{align}
         \min_\scriptE \quad & {\rm{Vol}}(\scriptE) \nonumber\\
        s.t. \quad & \Pr[y \in \scriptE] \ge \eta
    \end{align}
    The optimal argument is 
    \begin{equation}
       \scriptE(\mu, F^{-1}_{\chi^2_n}(\eta) \cdot \Sigma) 
    \end{equation}
    where $F^{-1}_{\chi^2_n}(\cdot)$ is the inverse chi-square cdf with $n$ degrees of freedom, and the optimal value is $F^{-1}_{\chi^2_n}(\eta)^n \cdot {\rm{Vol}}(\Sigma)$.
\end{lemma}
\begin{proof}
    Assume an optimal solution $\scriptE(\hat\mu, \hat\Sigma)$ and let $\kappa$ satisfy ${\rm{Vol}}(\kappa \Sigma) = {\rm{Vol}}(\hat\Sigma)$. %Denote $p(y) \= \scriptN(y;\mu,\Sigma)$ The probability density function of Gaussian distribution with $\mu$ and $\Sigma$ parameters.
    We only care about the symmetric difference, so denote:
    \begin{align}
        S_- &\= \scriptE(\mu, \kappa \Sigma) \setminus ( \scriptE(\hat{\mu}, \hat{\Sigma}) \cap \scriptE(\mu, \kappa \Sigma) )\nonumber \\
        S_+ &\= \scriptE(\hat{\mu}, \hat{\Sigma}) \setminus ( \scriptE(\hat{\mu}, \hat{\Sigma}) \cap\scriptE(\mu, \kappa \Sigma) )
    \end{align}
    By the choice of $\kappa$ we get:
    \begin{equation}
         {\rm{Vol}}(\kappa \Sigma) = {\rm{Vol}}(\hat{\Sigma}) \;\implies\; S \= {\rm{Vol}}(S_-) = {\rm{Vol}}(S_+)
    \end{equation}    
    The Gaussian density has an equal value on the edge of $\scriptE(\mu, \kappa \Sigma)$, has a larger value inside, and has a smaller value outside. Therefore,
    \begin{equation}
        p(y) \ge p \quad \forall y \in S_- \;,\quad
        p(y) < p \quad \forall y \in S_+
    \end{equation}
    where 
        \begin{equation}
        p \= \min_{y \in \scriptE(\mu, \kappa \Sigma)} p(y).
    \end{equation}
    Thus
    \begin{align}
        &\int_{y \in \scriptE(\mu, \kappa \Sigma)} p(y) \d{y} - \int_{y \in \scriptE(\hat{\mu}, \hat{\Sigma})} p(y) \d{y}\nonumber \\
        &= \int_{y \in S_-} p(y) \d{y} - \int_{y \in S_+} p(y) \d{y}\nonumber \\
        &\ge S \cdot \left(\inf_{y \in S_-} p(y) - \sup_{y \in S_+} p(y) \right) 
       \nonumber \\ &\ge S \cdot (p - p) = 0
    \end{align}
    with equality if and only if all parameters are equal. This means that $S_- = S_+ = \emptyset$ and the optimal parameters are $\mu=\hat\mu$ and $\kappa \Sigma=\hat\Sigma$.

    Finally, it remains to choose the minimum scaling that satisfies the constraint, namely $\kappa = F^{-1}_{\chi^2_n}(\eta)$.
\end{proof}

Next, we turn to a more practical estimation problem with features $x$ and unknown labels $y$.

\begin{theorem}
    \label{theorem: gaussian optimality}
    Let $x$ and $y$ be jointly Gaussian, then the optimal solution to \hyperref[eq: MVE]{(MVE)} is
    \begin{align}\label{gaussian}
      \mu(x) &= \E{[y|x]}% = \E[y] + C_{yx}C_{xx}^{-1}[x-\E[x]]
      \nonumber\\
      C(x) &= \kappa(x) \cdot  \E[(y-\mu(x))(y-\mu(x))^T] %\\&= \kappa(x) \cdot (C_{yy}-C_{yx}C_{xx}^{-1}C_{xy})
    \end{align}
    $\kappa(x) > 0$ is a scaling factor that satisfies the MVE coverage constraint. 
\end{theorem}
\begin{proof}
The conditional Gaussian distribution is given by
\begin{equation}
  p(y|x) = \scriptN(y;\mu(x), C_{y|x})  
\end{equation}
where $C_{y|x}$ does not depend on $x$. Problem \hyperref[eq: MVE]{(MVE)} can be expressed as
\begin{align}
    \min_{\scriptE(x)} \quad& \int_{\scriptX} p(x) \cdot {\rm{Vol}}_{\scriptE}((x)) \d{x} \nonumber\\
    \rm{s.t.} \quad& \int_\scriptX p(x) \cdot \int_{y \in \scriptE(x)} p(y|x) \d{y} \d{x} \ge \eta
\end{align}
Introducing auxiliary variables, we can rewrite the constraint
\begin{equation}
    \left\{ \begin{aligned}
    & \Pr[y \in \scriptE|x]  \ge \overline\eta(x) \quad \forall \; x\in \scriptX\\
    & \int_\scriptX p(x) \cdot \overline\eta(x) \d{x} \ge \eta
\end{aligned} \right.
\end{equation}
We can apply Lemma \ref{lemma:positive_proportion} for each $x$ separately, and its solution will always be the same $C_{y|x}$ times a scalar factor that depends on $x$. Together, the optimization problem is equal to
\begin{align}
    \min_{\overline\eta(x)} \quad& \int_\scriptX p(x) \cdot F_{\chi^2_n}^{-1}(\overline\eta(x))^n \d{x} \cdot {\rm{Vol}}(C_{y|x}) \nonumber\\
    \rm{s.t.} \quad& \int_\scriptX p(x) \cdot \overline\eta(x) \d{x} \ge \eta.
\end{align} 
\end{proof}

%\inote{The scaling $\kappa(x)$ ($=F_{\chi^2_n}^{-1}(\overline\eta(x))$) introduced in Theorem \ref{theorem: gaussian optimality} suggests that not necessarily one global scaling is optimal. The intuition is the opposite, as the inverse chi-square cdf grows much faster than $\bar\eta$, which is limited between zero and one; the preference is for smaller $\bar\eta$ for values of $x$, which are closer to the mean.}

The theorem formally shows the relation between uncertainty regions and covariance estimation in the joint Gaussian distribution. In this simplistic case, the shape of the ellipsoids is defined by the conditional covariances. These are independent of the conditioning, and the dependence on $x$ is only through the calibration factor. This is not true in general. In other distributions, there is no clear relation between the shapes and the covariances of the error and the shapes typically depend on the values of $x$. The next section proposes a neural network solution that addresses these more practical settings. 

\section{Learning Confidence Ellipsoid Nets}
\label{section: nets}

In this section, we propose LMVE, a neural network-based solution to \hyperref[eq: MVE]{(MVE)}. 
We rely on a standard neural architecture with fully connected layers to ensure rich expressibility, generalization, and low computational complexity at inference time.

LMVE is defined as
\begin{align}
    C_\theta(x) &= R_\theta(x)^T R_\theta(x) + \epsilon I \nonumber\\
    R_\theta(x) &= L_3 \cdot \sigma(L_2 \cdot \sigma(L_1 \cdot x + b_1) + b_2) + b_3
\end{align}
where
$L_1 \in \doubleR^{4d \times d}$, $L_2 \in \doubleR^{d \times 4d}$, $L_3 \in \doubleR^{n^2 \times d}$, $b_1 \in \doubleR^{4d}$, $b_2 \in \doubleR^{d}$, $b_3 \in \doubleR^{n^2}$ and $\theta = (L_1, L_2, L_3, b_1, b_2, b_3)$.
The operator $\sigma$ is the ReLU activation function with dropout, and $\epsilon>0$ is a regularization parameter that ensures positive definiteness.

Problem \hyperref[eq: MVE]{(MVE)} is challenging for learning due to its non-differentiability and strict constraint. Therefore, LMVE is based on three phases: initialization, training, and calibration. We divide the available data into two subsets, samples $z_1,\cdots,z_{m_t}$ for the initialization and training (that is, training set) and samples $z_{m_t+1},\cdots,z_{m_t+m_c}$ for the calibration (namely validation set), where $m=m_t+m_c$.

First, in the initialization phase, we imitate an existing baseline. We use the baseline outputs as approximate labels and train our network to approximate them using a standard Mean Squared Error (MSE) loss. The baseline and the initialization phase only use samples from the training set.

Second, in the training phase, we consider \hyperref[eq: MVE]{(MVE)} using its Lagrange penalized form directly
\begin{align}
     \min_{\theta} \, \Pr\left[y\notin {\scriptE}(\mu(x),C_\theta(x))\right] +\lambda \E[{\rm{Vol}}_\theta(x)]
\end{align}
We replace the probability and expectation with empirical averages over the training set. We tried different smooth surrogates for the probability function but found that the best results were obtained by minimizing the average Mahalanobis distance. Together, the training loss is 
\begin{align}
    \min_{\theta} \frac 1{T} \sum_{i=1}^{T} \ell(x_i, y_i ; \theta) \nonumber
\end{align}
\begin{align}
  \ell(x, y ; \theta) \= M_\theta(x,y) + \lambda \det{({C}_\theta(x))^\frac12}\nonumber
\end{align}
\begin{equation}\label{eq:def:unc}
    M_\theta(x,y) \= (y - \mu(x))^T{C_\theta}^{-1}(x)(y - \mu(x))
\end{equation}
where $\lambda$ is a hyperparameter that scales between the volume's importance and accuracy. To tune it, we rely again on the existing calibrated baseline and choose lambda to balance the two terms:
\begin{equation}
    \lambda = \frac{
        \frac1{T} \sum_{j=1}^T M_B(x_j,y_j)
    }{
        \frac1{T} \sum_{j=1}^T {\det({C}_B(x_j))}
    }
\end{equation}
where the ellipsoids are obtained from the calibrated baseline (denoted by subscript $B$). We emphasize that the baseline is trained and calibrated only using the training set to avoid dependency on the validation set. 

Third, to ensure the strict coverage constraints, we use a calibration step using the theory of conformal prediction \cite{messoudi2022ellipsoidal}. We evaluate our network on the validation samples and compute their Mahalanobis distances
\begin{align}
    \alpha_i = M_\theta(x_i, y_i)\qquad i=m_t,\cdots,m_t+m_c
\end{align}
 We sort these distances and define $\alpha_q$ as the $q$'th largest value where 
\begin{equation}
    q=\frac{\ceil{(m_c + 1)\eta}}{m_c}
\end{equation}
We then rescale the ellipsoids by $\alpha_q$ so that the final LMVE shapes are defined as 
\begin{align}\label{clibrLMVE}
    C_{LMVE}(x)=\alpha_q \cdot {C}_\theta(x)
\end{align}
This ensures that 
\begin{equation}
\label{eq:conformal}
\Pr[y \in \scriptE_{LMVE}(x)] \ge \eta
\end{equation}
as required.

\section{Experiments}

\begin{table*}[t]
\centering
\caption{Results over 90\% calibration.}
\begin{tabular}{ccccc} 
 \hline
  & GE & NLE & LMVE & Copula \\ 
 \hline
 ble\_rssi & $89.8$ ($\pm 4.8$) & $90.1$ ($\pm 3.1$) & $89.6$ ($\pm 4.3$) & $88.8$ ($\pm 3.7$) \\
 enb & $88.2$ ($\pm 5.2$) & $89.8$ ($\pm 4.2$) & $89.3$ ($\pm 4.8$) & $88.9$ ($\pm 5.4$) \\
 indoor\_localization & $89.8$ ($\pm 1.0$) & $89.7$ ($\pm 1.2$) & $89.6$ ($\pm 1.0$) & $89.8$ ($\pm 1.0$) \\
 residential\_building & $88.6$ ($\pm 7.0$) & $88.5$ ($\pm 6.5$) & $88.7$ ($\pm 7.6$) & $83.9$ ($\pm 6.6$) \\
 % Data & $ $ ($\pm $) & $ $ ($\pm $) & $ $ ($\pm $) \\ 
 \hline
 \hline
 ble\_rssi & $10.2$ ($\pm 1.5$) & $9.6$ ($\pm 1.0$) & \boldmath $7.8$ ($\pm 1.8$) & $54.8$ ($\pm 18.3$) \\
 enb & $34.0$ ($\pm 7.6$) & $18.3$ ($\pm 3.4$) & \boldmath $17.90$ ($\pm 1.1$) & $263.9$ ($\pm 118.4$) \\
 indoor\_localization & $7.0 \cdot 10^{3}$ ($\pm 178.6$) & $4.3 \cdot 10^{3}$ ($\pm 136.5$) & \boldmath $2.0 \cdot 10^{3}$ ($\pm 51.6$) & $22.4 \cdot 10^{3}$ ($\pm 1355$) \\
 residential\_building & $2.2 \cdot 10^{5}$ ($\pm 15.4 \cdot 10^{4}$) & $1.3 \cdot 10^{5}$ ($\pm 7.8 \cdot 10^{4}$) & \boldmath $1.1\cdot 10^{5}$ ($\pm 6.8 \cdot 10^{4}$) & $69.9 \cdot 10^{5}$ ($\pm 6603 \cdot 10^{4}$) \\
 % Data & $ \cdot 10^{-1}$ ($\pm  \cdot 10^{-1}$) & $ \cdot 10^{-1}$ ($\pm  \cdot 10^{-1}$) & $ \cdot 10^{-1}$ ($\pm  \cdot 10^{-1}$) \\
 % \boldmath
 \hline
\end{tabular}
\label{table: comparing 90}
\end{table*}
% \begin{table*}[t]
% \centering
% \begin{tabular}{cccc} 
%  \hline
%   & GE & NLE & LMVE \\ 
%  \hline
%  ble\_rssi & $80.000$ ($\pm 5.427$) & $80.310$ ($\pm 5.214$) & $79.986$ ($\pm 5.048$) \\
%  enb & $79.766$ ($\pm 7.099$) & $81.117$ ($\pm 6.098$) & $78.883$ ($\pm 6.826$) \\
%  indoor\_localization & $79.766$ ($\pm 1.207$) & $79.713$ ($\pm 1.739$) & $79.701$ ($\pm 1.492$) \\
%  residential\_building & $79.947$ ($\pm 9.099$) & $80.421$ ($\pm 9.257$) & $80.474$ ($\pm 9.377$) \\
%  % Data & $ $ ($\pm $) & $ $ ($\pm $) & $ $ ($\pm $) \\ 
%  \hline
%  \hline
%  ble\_rssi & $6.712 \cdot 10^{0}$ ($\pm 9.059 \cdot 10^{-1}$) & $6.620 \cdot 10^{0}$ ($\pm 8.335 \cdot 10^{-1}$) & \boldmath $5.161 \cdot 10^{0}$ ($\pm 5.737 \cdot 10^{-1}$) \\
%  enb & $1.928 \cdot 10^{1}$ ($\pm 6.271 \cdot 10^{0}$) & $1.412 \cdot 10^{1}$ ($\pm 2.181 \cdot 10^{0}$) & \boldmath $9.232 \cdot 10^{0}$ ($\pm 5.822 \cdot 10^{0}$) \\
%  indoor\_localization & $5.681 \cdot 10^{3}$ ($\pm 1.037 \cdot 10^{2}$) & $3.455 \cdot 10^{3}$ ($\pm 9.939 \cdot 10^{1}$) & \boldmath $2.027 \cdot 10^{3}$ ($\pm 4.150 \cdot 10^{1}$) \\
%  residential\_building & $8.487 \cdot 10^{4}$ ($\pm 5.041 \cdot 10^{4}$) & $8.356 \cdot 10^{4}$ ($\pm 3.481 \cdot 10^{4}$) & \boldmath $6.751 \cdot 10^{4}$ ($\pm 3.034 \cdot 10^{4}$) \\
%  % Data & $ \cdot 10^{-1}$ ($\pm  \cdot 10^{-1}$) & $ \cdot 10^{-1}$ ($\pm  \cdot 10^{-1}$) & $ \cdot 10^{-1}$ ($\pm  \cdot 10^{-1}$) \\ 
%  % \boldmath
%  \hline
% \end{tabular}
% \caption{Results over 80\% calibration.}
% \label{table: comparing 80}
% \end{table*}

In this section, we compare the algorithms on different real-world datasets.

Algorithms: In all the methods, the centers of the ellipsoids $\mu$ are computed identically using a multi-output Support Vector Regressor (SVR) fitted on the training set. We compare three algorithms for generating ellipsoids. The shapes $C(x)$ are fitted on the training set and then rescaled to satisfy the coverage constraint on the validation set as defined in \ref{clibrLMVE}. The algorithms are:
\begin{itemize}
    \item Gaussian Ellipsoid (GE): the sample version of (\ref{gaussian}) defined as
    \begin{equation}
        \hat{C}_{\rm GE}(x) = \frac1{T} \sum_{i=1}^T (y_i - \mu_i) (y_i - \mu_i)^T
    \end{equation}
    We approximate the conditional mean using SVR for consistency with the other algorithms.
    \item NLE \cite{messoudi2022ellipsoidal}: local sample covariance matrices of the $5\%$ nearest neighbors $N(x)$ of each $x$ with a small GE regularization
    \begin{multline}
        \hat{C}_{\rm NLE}(x) = 0.95 \cdot \frac1{|N(x)|} \sum_{i \in N(x)} (y_i - \mu_i) (y_i - \mu_i)^T \\+ 0.05 \cdot \hat{C}_{\rm GE}
    \end{multline}
    \item Copula \cite{messoudi2021copula}: We also compare to the empirical copula method. It provides $\alpha_s$ that promises $\Pr\[\frac{|y_i - \hat{y}_i|}{\sigma_i} \le \alpha_s^{(i)}\,\, \forall i \in [d]\] \ge \eta$ for a new sample $y$. We used the same net structure that was suggested in the article. 
    \item LMVE: $C_{LMVE}(x)$ is the neural network output where hyperparameters are chosen to maximize performance on the validation set. Its initialization is based on the NLE baseline. 

\end{itemize}

Datasets: We compare four real-world localization datasets, as detailed in Table \ref{table: datasets}. The full data process, which explains the dimension difference, can be found on GitHub.

\begin{table}[h!]
    \centering
    \caption{Datasets}
    \begin{tabular}{ccccc}
        \hline
        Name & Size & Input & Output & Source \\
        \hline
        ble\_rssi & 1,420 & 13 & 2 & \cite{mohammadi2017semisupervised} \\
        enb & 768 & 8 & 2 & \cite{tsoumakas2011mulan} \\
        indoor\_localization & 19,937 & 519 & 2 & \cite{torres2014ujiindoorloc} \\
        residential\_building & 372 & 103 & 2 & \cite{rafiei2016novel} \\
        \hline
    \end{tabular}
    \label{table: datasets}
\end{table}

Each dataset was divided into three parts: a training set (81\%), a validation set (9\%), and a test set (10\%). We conducted 50 experiments and reported the average performance and standard deviation.

Results for $90\%$ coverage are in Table \ref{table: comparing 90}. The top section shows preferred coverage values close to 90\%. The bottom section displays mean volume values, with smaller values indicating better performance. Strong calibration makes the methods accurate, but LMVE improves average volume. GE is the cheapest for storage and computation. NLE is expensive for large datasets. LMVE is efficient with a simple three-layer neural network. Further comparison is in Table \ref{table: complexity}.

\begin{table}[h!]
    \centering
    \caption{Computational complexity in $\scriptO$ notion. $N$ is the number of neighbors in the NLE. $L$ is the hidden layer size.}
    \begin{tabular}{cccc}
        \hline
        & Memory & Inference Time \\
        \hline
        GE & $d^2$ & const \\
        NLE & $m_t (n + d)$ & $ N \cdot (d \log(m_t) + n^2)$ \\
        LMVE & $d^2 + n^2$ & $d^2 + n^2$ \\
        Copula & $dL + L^2 + Ln$ & $dL + L^2 + Ln$ \\
        \hline
    \end{tabular}
    \label{table: complexity}
\end{table}

We tested more complex losses, such as hinge loss and robust generalized Gaussian-based losses, that smooth the probability function. However, they did not lead to significant improvements in some datasets and were therefore omitted.

\begin{table}[t]
\centering
\caption{Results over 90\% calibration on ENB \cite{tsoumakas2011mulan} dataset. The results are compared to accuracy of $89.325$ ($\pm 4.883$) and expected volume of $17.90$ ($\pm 1.184$).}
\label{table: ablation enb}
\begin{tabular}{cc} 
 \hline
  & LMVE \\ 
 \hline
 Without smart init & $88.3$ ($\pm 5.1$) \\ 
 Fixed $\lambda=1$ & $88.8$ ($\pm 4.7$) \\ 
 2 Layers & $87.7$ ($\pm 4.7$) \\ 
 50\% Validation & $89$ ($\pm 4.1$) \\ 
 \hline
 \hline
 Without smart init & $18.23$ ($\pm 51.5$) \\
 Fixed $\lambda=1$ & $11.40$ ($\pm 4.3$) \\
 2 Layers & $16.58$ ($\pm 7.7$) \\
 50\% Validation & $17.37$ ($\pm 10.5$) \\
 % \boldmath
 \hline
\end{tabular}
% \centering
% \caption{Results over 90\% calibration on BLE\_RSSI \cite{mohammadi2017semisupervised} dataset. The results are compared to accuracy of $89.634$ ($\pm 4.306$) and expected volume of $7.863$ ($\pm 1.893$).}
% \begin{tabular}{cc} 
%  \hline
%   & LMVE \\ 
%  \hline
%  Without smart init & $89.606$ ($\pm 4.008$) \\ 
%  Fixed $\lambda=1$ & $90.225$ ($\pm 3.960$) \\ 
%  2 Layers & $90.113$ ($\pm 4.667$) \\ 
%  50\% Validation & $89.338$ ($\pm 2.781$) \\ 
%  \hline
%  \hline
%  Without smart init & $8.113$ ($\pm 1.260$) \\
%  Fixed $\lambda=1$ & $8.040$ ($\pm 1.469$) \\
%  2 Layers & $8.043$ ($\pm 1.324$) \\
%  50\% Validation & $7.979$ ($\pm 0.278$) \\
%  % \boldmath
%  \hline
% \end{tabular}
% \label{table: ablation ble}
\end{table}

We conducted ablation studies on a few datasets to comprehend the contribution of each component to the overall system. However, we only report some results in Table \ref{table: ablation enb} due to space limitations. The key conclusions are as follows: Smart initialization enhances the performance. Reducing the number of layers in architecture leads to a decline in coverage. The effect of lambda value is less significant and depends on the dataset. Finally, as expected, using a larger validation set improves the coverage but results in inaccuracies in the volume.

\section{Discussion}
This paper proposes a scalable deep learning framework called the LMVE neural network. This framework outputs uncertainty ellipsoids around predictions. LMVE combines CELLO and NLE and improves upon them by providing modern and efficient data modeling.

LMVE is an important deep-learning technique for estimating uncertainty ellipsoids, but many questions remain regarding its usage. Regarding problem formulation, LMVE only considers the average coverage of random labels. Extending its application to confidence regions around deterministic unknown labels would be interesting. Training LMVE is challenging, so future work should focus on improving initialization methods, optimization algorithms, and hyperparameter tuning approaches.

\section*{Appendix: Technical details on LMVE}
We fitted the models in PyTorch using dropout parameters of $0.1$ or $0.5$, and $200,000$ iterations of an Adam optimizer (half for initialization and half for training) with rates $\{ 10^{-3}, 10^{-5} \}$.

\section*{Acknowledge}
We want to thank Erez Peterfreund, Roy Friedman, and Tzvi Diskin for the consulting.

\bibliographystyle{IEEEtran}
\bibliography{main}

\end{document}